\documentclass[final,12pt]{colt2022} % Anonymized submission 
\title[Universal Online Learning with Bounded Loss: Reduction to Binary Classification]{Universal Online Learning with Bounded Loss: \\Reduction to Binary Classification}

\usepackage{times}
\usepackage[mathscr]{eucal}
\coltauthor{%
 \Name{Mo\"ise Blanchard} \Email{moiseb@mit.edu}\\
 \addr Massachusetts Institute of Technology
 \AND
 \Name{Romain Cosson} \Email{cosson@mit.edu}\\
 \addr  Massachusetts Institute of Technology
}

\newcommand{\comment}[1]{}

\newcommand{\suol}{\text{SUOL}}

\newcommand{\smv}{{\text{SMV}}}

\newcommand{\nn}{\text{NN}}

\newcommand{\Acal}{\mathcal{A}}
\newcommand{\Bcal}{\mathcal{B}}

\newcommand{\Ecal}{\mathcal{E}}
\newcommand{\Fcal}{\mathcal{F}}

\newcommand{\Lcal}{\mathcal{L}}

\newcommand{\Pcal}{\mathcal{P}}

\newcommand{\Xcal}{\mathcal{X}}
\newcommand{\Ycal}{\mathcal{Y}}

\newcommand{\Ebb}{\mathbb{E}}
\newcommand{\Nbb}{\mathbb{N}}
\newcommand{\Zbb}{\mathbb{Z}}
\newcommand{\Pbb}{\mathbb{P}}

\newcommand{\Rbb}{\mathbb{R}}

\newcommand{\Xbb}{\mathbb{X}}
\newcommand{\Ybb}{\mathbb{Y}}

\usepackage{bbm}

\newcommand{\1}{\mathbbm{1}}%{{\rm 1}\kern-0.24em{\rm I}}

\definecolor{dark_red}{rgb}{0.2,0,0}

\DeclareMathOperator*{\argmax}{arg\,max}
\DeclareMathOperator*{\argmin}{arg\,min}

\usepackage{graphicx}

\renewenvironment{proof}[1][]{\par\noindent{\bf Proof #1\ }}{\hfill\BlackBox\\[2mm]}

\begin{document}

\maketitle

%\begin{abstract}%
%We study universal online learning with non-i.i.d. processes when the loss function is bounded. The notion of an \emph{optimistically universal learning rule} was defined in an effort to study learning theory under minimal assumptions \cite{hanneke2021learning}. When the loss is unbounded, the simple \emph{memorization} learning rule is always optimistically universal \cite{blanchard2022universal}. However, when the loss function is bounded, the existence of optimistically universal rule and characterization of learnable processes could depend on the setting \cite{hanneke2021open}. In this paper, we show that both questions are invariant to the choice of output space, thereby closing an open question of \cite{hanneke2021learning} (Open Problem 3). Specifically, we show that the set of processes that admit universal online learning is the same for binary classification as for multiclass classification with countable number of classes. Consequently, any output setting with bounded loss can be reduced to binary classification. Our reduction is constructive and practical. Indeed, we show that the \emph{nearest neighbor} algorithm is transported by our construction. We also prove that nearest neighbor successfully learns finite unions of intervals when the data generating process admits strong universal learning.  %\end{abstract}

\begin{abstract}%
%\rc{This is a tentative, do not delete previous abstract. - I prefer first version at this stage}
We study universal consistency of non-i.i.d. processes in the context of online learning. A stochastic process is said to admit \emph{universal consistency} if there exists a learner that achieves vanishing average loss for any measurable response function on this process. When the loss function is unbounded, \cite{blanchard2022universal} showed that the only processes admitting strong universal consistency are those taking a finite number of values almost surely. However, when the loss function is bounded, the class of processes admitting strong universal consistency is much richer and its characterization could be dependent on the response setting \citep{hanneke2021open}. In this paper, we show that this class of processes is independent from the response setting thereby closing an open question of \cite{hanneke2021learning} (Open Problem 3). Specifically, we show that the class of processes that admit universal online learning is the same for binary classification as for multiclass classification with countable number of classes. Consequently, any output setting with bounded loss can be reduced to binary classification. Our reduction is constructive and practical. Indeed, we show that the \emph{nearest neighbor} algorithm is transported by our construction. For binary classification on a process admitting strong universal learning, we prove that nearest neighbor successfully learns at least all finite unions of intervals. 
\end{abstract}


\section{Introduction}
\paragraph{Problem setup.}
%This work builds upon the stream of papers on universal online learning \cite{hanneke2021learning,hanneke2021open,blanchard2022universal}, which aim to study the classical question of \emph{learnability} under minimal assumptions. We briefly recall the context here and formally introduce all the concepts and notations in Section \ref{subsec:formal_setup}.
We consider the online learning framework where we observe a sequence of input points $\Xbb = (X_t)_{t\geq 1}$ from a separable metric \textit{instance space} $(\Xcal,\rho)$ and the associated target values $\Ybb = (Y_t)_{t\geq 1}$ from some separable near-metric \textit{value space} $(\Ycal,\ell)$. We assume that the output data stream $\Ybb$ is generated from $\Xbb$ in a noiseless fashion through an unknown function $f^*:\Xcal\to\Ycal$, i.e. we have $Y_t = f^*(X_t)$ for all $t\geq 1$. Learning occurs sequentially: at a time step $t\geq 1$, the learner observes a new input $X_t$ (covariates) and outputs a prediction $\hat Y_t$ based on the historical data $(\Xbb_{\leq t-1},\Ybb_{\leq t-1})$. The performance of the learning rule used by the learner is measured by the long-run average loss $\frac{1}{T}\sum_{t=1}^T\ell(Y_t,\hat{Y}_t)$. We say that a learning rule is \emph{universally consistent} under $\Xbb$ if the long-run average loss converges to $0$ almost surely, for any measurable function $f^*$. If such a learning rule exists, we say that $\Xbb$ admits strong universal online learning. Following \cite{hanneke2021learning}, we are interested in the set $\suol$ containing processes $\Xbb$ that admit strong universal online learning. A priori, $\suol$  may depend on the setup $(\Xcal,\rho, \Ycal,\ell)$ so we may specify $\suol_{(\Xcal,\rho,\Ycal, \ell)}$. 

\paragraph{Prior work and motivations.} In universal learning, the goal is to design learning rules that are consistent with a large variety of data generating processes $(\Xbb,\Ybb)$. A celebrated example, \citet{stone1977consistent,devroye1994strong} show that the $k$-nearest neighbour learning rule with $k/\log t\rightarrow \infty$ and $k/t\rightarrow 0$ is consistent with any i.i.d process under mild hypothesis. More recently, \cite{hanneke2021universal,gyorfi2021universal,cohen2022metric} gave algorithms that are consistent for any i.i.d. process for any metric space $\Xcal$ that admits such an algorithm. In this paper, we primarily focus on \emph{strong} consistency, where we ask the average loss to decay to zero almost surely \citep{gordon1978asymptotically}. The literature has also widely investigated \emph{weak} consistency \citep{muller1987weak}, where convergence is in expectation. If randomness or noise is allowed in $f^*$, consistency is attained when the average loss converges to the loss corresponding to the best deterministic function, i.e. the Bayes loss \citep{stone1977consistent,devroye2013probabilistic}. For this reason, universal consistency is sometimes referred to as Bayes risk efficiency \citep{gordon1978asymptotically}. For simplicity, this paper assumes that $f^*$ is noiseless and rather focuses on relaxing the assumptions on the input process $\Xbb$. 

Indeed, most of the work on universal learning requires the input $\Xbb$ to be drawn i.i.d from a joint distribution \citep{stone1977consistent,haussler1994predicting, hanneke2021universal}. Alternatively it is asked to be stationary ergodic \citep{morvai1996nonparametric,gyorfi:07,gyofi2002strategies}, to satisfy a law of large numbers \citep{morvai1996nonparametric,steinwart2009learning} or to admit convergent relative frequencies \citep{hanneke2021learning}. Another line of work, \citep{littlestone1988learning,ben2009agnostic} makes no assumption on the input data stream $\Xbb$ but restricts the hypothesis class to functions $f^*$, e.g. to functions admitting finite Littlestone dimension. Many other setups have been considered, mixing restrictions on the pair $(\Xbb,f^*)$ \citep{ryabko2006pattern,urner2013probabilistic,bousquet2021theory}.

Following the work of \cite{hanneke2021learning}, we make no assumption on the input data $\Xbb$ other than the fact that it is a stochastic process. We are particularly interested in the set $\suol_{(\Xcal,\rho,\Ycal,\ell)}$ of processes $\Xbb$ that admit strong universal online learning, i.e. such that there exists a learner which achieves vanishing average loss for any choice of measurable function $f^*:\Xcal\rightarrow \Ycal$. When the loss function is unbounded, i.e. $\sup_{y_1,y_2}\ell(y_1,y_2) =\infty$, this set contains exactly the processes that take a finite number of values almost surely \citep{blanchard2022universal} and is therefore independent of the value space $(\Ycal,\ell)$. When the loss function is bounded, i.e. $\sup_{y_1,y_2}\ell(y_1,y_2) <\infty$, \citet{hanneke2021learning} conjectured that such processes are characterized by a simple condition that we call $\smv$, standing for sublinear measurable visits, which is also independent of the setting. He posed as an open question whether $\suol_{(\Xcal,\rho,\Ycal,\ell)}$ would depend on the setting $(\Ycal,\ell)$ subject to the loss being bounded (\citet{hanneke2021learning}, Open Problem 3).  

One interest of characterizing the set $\suol$ is to identify learning rules which are universally consistent for all processes in $\suol$, i.e. that achieve universal consistency whenever it is possible \citep{hanneke2021learning}. These \emph{optimistically universal} learning rules enjoy the convenient property that if they fail to achieve universal learning for a specific input process $\Xbb$, any other online learning rule would fail as well. For unbounded loss, the simple \emph{memorization} learning rule was shown to be optimistically universal \citep{blanchard2022universal} for any setting $(\Ycal,\ell)$. For bounded loss, an important question---very related to (Open Problem 3 \cite{hanneke2021learning})---is whether the existence of an optimistically universal learning rule depends on the setting $(\Ycal, \ell)$.

%\rc{Maybe add a paragraph on boundedness ? Maybe add a paragraph to explain that the simplest setting is binary classification ?} \rc{Maybe add a paragraph to distinguish inductive from online?}

\paragraph{Contributions.} We close a conjecture formulated in \cite{hanneke2021learning} by showing that the set of universally learnable sequences $\suol$ is invariant with respect to the setting $(\Ycal,\ell)$ when the loss is bounded. Precisely, we show that any learning task can be reduced to the binary classification setting $(\{0,1\},\ell_{01})$ where $\ell_{01}$ is the binary indicator loss. Our main result is stated as follows.

\begin{theorem}
\label{thm:invariance_suol}
For any separable near-metric space $(\Ycal,\ell)$ with $0<\bar \ell <\infty$, we have $\textnormal \suol_{(\Xcal,\rho,\Ycal,\ell)}=\textnormal\suol_{(\Xcal,\rho,\{0,1\},\ell_{01})}.$
\end{theorem}

\noindent
This shows that to characterize the set $\suol$ it suffices to focus on universal binary classification. Our work builds upon \cite{hanneke2021learning} which proves that universal learning can be reduced to either binary classification $(\{0,1\},\ell_{01})$ or \emph{multiclass} classification with countable number of labels $(\Nbb,\ell_{01})$. Thus, we show the invariance of $\suol$ to the learning setting by proving that universal binary classification and universal countably-many classes classification are equivalent. Further, our proof is constructive and therefore would provide a construction of an optimistically universal learning rule for any setting $(\Ycal,\ell)$ given an optimistically universal learning rule for binary classification---if such learning rule exists.

\begin{theorem}
\label{thm:invariance_optimistic}
The existence of an optimistically universal learning rule is invariant to the output space $(\Ycal,\ell)$ when $0<\bar \ell <\infty$. In particular, provided an optimistically universal learning rule for binary classification $(\{0,1\},\ell_{01})$ one can construct an optimistically universal learning rule for a general setup $(\Ycal,\ell)$ with $0<\bar \ell <\infty$.
\end{theorem}

\noindent
Last, we make practical use of this construction to analyze the simple \emph{nearest neighbour} learning rule. In the restricted setting $\Xcal = \Rbb$ we show that for processes that admit strong universal learning, the nearest neighbour learning rule successfully learns functions $f^*:\Rbb\to\{0,1\}$ which represent finite union of intervals i.e. is capable of solving simple classification tasks.

\paragraph{Outline of the paper.} The paper is organized as follows. In Section \ref{sec:prelim} we formally introduce the universal online learning setup and recall some useful results from \citet{hanneke2021learning}. We then prove the main reduction theorems and present a class of learning rules that are preserved by this reduction in Section \ref{sec: result reduction binary}. This class includes for instance the nearest neighbor rule. Finally, we focus on this learning rule in Section \ref{sec:nearest_neighbour} proving that is consistent for simple classification tasks.

\paragraph{Notations.} In the following, $\ell_{01}$ will denote the indicator loss function $\ell_{01}(i,j)=\1({i\neq j})$ irrespective of the output space $\Ycal$. Note that it satisfies the relaxed triangle inequality with constant $c_\ell=1$. When the space $\Xcal$ is clear from the context, we simplify the notation $\suol_{(\Xcal,\Ycal,\ell)}$ to $\suol_{(\Ycal,\ell)}$. We might also omit the loss function $\ell$ when there is no ambiguity.

\section{Background and Preliminaries}\label{sec:prelim}
\subsection{Formal Setup}\label{subsec:formal_setup}
\paragraph{Instance and value space.} Recall that the sequence of inputs $\Xbb = (X_t)_{t\geq 1}$ comes from a separable metric \textit{instance space} $(\Xcal,\rho)$ and the targets $\Ybb = (Y_t)_{t\geq 1}$ belong to some separable near-metric \textit{value space} $(\Ycal,\ell)$. The near-metric loss function $\ell:\Ycal^2\to[0,\infty)$ is assumed to satisfy symmetry $\ell(y_1,y_2) =\ell(y_2,y_1)$, discernibly $\ell(y_1,y_2)=0$ if and only if $y_1=y_2$, as well as a relaxed triangle inequality $\forall y_1, y_2, y_3\in \Ycal^3: \ell(y_1,y_3)\leq c_{\ell} (\ell(y_2,y_1)+\ell(y_2,y_3))$, where $c_{\ell}$ is a finite constant. 
For instance, the squared loss that is classically used in regression settings satisfies this identity with $c_{\ell} = 2$. 
In the following, we will denote by $\bar \ell=\sup_{y_1,y_2 \in  \Ycal} \ell(y_1,y_2) $ the supremum of the loss function. In particular, the loss function is said to be \textit{bounded} when $\bar \ell<\infty$. 

%Last, we assume without loss of generality that the loss is not identically null. Indeed, if $\ell=0$, then the set of universally learnable input sequences $\suol$ is the complete set of stochastic processes on $\Xcal$ and any learning rule would be optimistically universal {\color{red} too early?}. 

\paragraph{Data generation process.} The stream of input points $\Xbb$ will be modeled as a general stochastic process with respect to the $\sigma$-algebra induced by the metric $\rho$ on $\Xcal$. This differs substantially from most of the statistical learning literature which often imposes additional hypothesis such as being i.i.d., or satisfying the law of large numbers. The stream of output data $\Ybb$ is assumed to be generated from $\Xbb$ in a noiseless fashion through an unknown fixed measurable function $f^*:\Xcal\to\Ycal$. Precisely, we have $Y_t = f^*(X_t)$ for all $t\geq 1$. When considering bounded time horizon $t\geq 1$, we will use the following notation: $\Xbb_{\leq t} = \{X_1,...,X_t\}$ and $\Xbb_{< t} = \{X_1,...,X_{t-1}\} $.

\paragraph{Online learning.} Formally, an online learning rule is defined as a sequence $f_{\boldsymbol{\cdot}} = \{f_t\}_{t=1}^{\infty}$ of measurable functions $f_t:\Xcal^{t-1}\times \Ycal^{t-1}\times\Xcal \rightarrow \Ycal$. Given $t-1$ training examples of the form $(X_i, f^*(X_i))\in \Xcal\times \Ycal$ and a new input sample $X_t\in \Xcal$, the online learning rule $f_t$ makes prediction $f_t(\Xbb_{< t}, \Ybb_{< t}, X_t)$ for $f^*(X_t)$. We wish to minimize the asymptotic loss,
\begin{equation*}
    \Lcal_{\Xbb}(f_{\boldsymbol{\cdot}},f^*) = \limsup_{T\to\infty}\frac{1}{T}\sum_{t=1}^{T}\ell(f_t(\Xbb_{< t},\Ybb_{< t}, X_{t}), f^*(X_t)).
\end{equation*}
We say that the online learning rule $f_{\boldsymbol{\cdot}}$ is consistent under the input process $\Xbb$ and for the target function $f^*$ if $\Lcal_{\Xbb}(f_{\boldsymbol{\cdot}},f^*)=0 ~~(a.s.)$.

\paragraph{Processes admitting strong universal online learning.}
We say that a stochastic process $\Xbb$ admit \emph{strong universal online learning} if there exists a learning rule $\{f_t\}_{t=1}^\infty$ that is consistent for all measurable target functions $f^*:\Xcal\to\Ycal$ on $\Xbb$. We denote by $\suol_{(\Xcal,\rho,\Ycal, \ell)}$ the set of all processes admitting strong universal online learning. Note that learning rules are allowed to depend on the process $\Xbb$. If a given learning rule is universally consistent under all processes in $\suol_{(\Xcal,\rho,\Ycal, \ell)}$ we say it is \emph{optimistically universal}.

\subsection{Comparing the general setting to binary and countable classification}
One of the main contributions of the paper is to show that the set $\suol$ of input processes $\Xbb$ admitting universal learning is invariant to the choice of value space subject to the loss being bounded. To do so, we compare  $\suol_{(\Ycal,\ell)}$ for different value spaces $(\Ycal,\ell)$. Specifically, to show that $\suol_{(\Ycal,\ell)}\subset \suol_{(\Ycal',\ell')}$, one aims to construct a universally consistent learning rule for $(\Ycal',\ell')$ from a universally consistent learning rule for $(\Ycal,\ell)$ under any fixed process $\Xbb\in \suol_{(\Ycal,\ell)}$. In this section, we recall two important known inclusions that hold for any bounded loss setup $(\Ycal,\ell)$. The first result compares the general setting to binary classification.

\begin{proposition}[\cite{hanneke2021learning}]
\label{prop:binary_izi}
For any separable near-metric space $(\Ycal,\ell)$ with $0<\bar \ell<\infty$, 
\begin{equation*}
    \textnormal \suol_{(\Ycal,\ell)}\subset \textnormal\suol_{(\{0,1\},\ell_{01})}.
\end{equation*}
\end{proposition}
This shows that binary classification is in essence the easiest setting: whenever universal online learning is achievable for some setting $(\Ycal,\ell)$, the learning rule that works on this setting should be able to perform binary classification (note that we simply require $\Ycal$ to contain at least two elements). A formal proof is given in Appendix \ref{ap:proof_binary_izi}, we note that it does not require the boundedness of $\ell$.

In the same spirit, we now recall that any process $\Xbb$ admitting strong universal online learning for countable classification $(\Nbb,\ell_{01})$ admits strong universal online learning on any separable value space $(\Ycal,\ell)$. Hence, countable classification is in essence the hardest setting. 
\begin{theorem}[\cite{hanneke2021learning}]
\label{thm:reduction_countable}
For any separable near-metric space $(\Ycal,\ell)$ with $0<\bar \ell<\infty$,
\begin{equation*}
    \textnormal\suol_{(\Nbb,\ell_{01})} \subset \textnormal\suol_{(\Ycal,\ell)}. 
\end{equation*}
\end{theorem}
A proof of this theorem is given in (\cite{hanneke2021learning} Theorem 45). It uses a number of intermediary lemmas that are not introduced in this paper. Instead, we provide novel arguments that greatly simplify the proof and that will have practical use in Section \ref{sec:nearest_neighbour}.\\

\begin{proof}
We fix a process $\Xbb \in \suol_{(\Nbb,\ell_{0,1})}$, and let $f^{\Nbb}_{\boldsymbol{\cdot}}$ be the corresponding strongly consistent learning rule. By separability, there exists a dense countable sequence $(y^i)_{i\geq 1}$ of $\Ycal$ i.e. such that $\forall y\in \Ycal: \inf_{i\in \Nbb}\ell(y^i,y)=0$. Following \cite{hanneke2021learning}, given a prediction task on $(\Ycal, \ell)$ and $\epsilon >0$, we reduce it to a countable classification using the function $h_\epsilon:y\in \Ycal\mapsto \inf\{i\in \Nbb: \ell(y^i,y)<\epsilon\} \in \Nbb$. This allows to define the $\epsilon$-learning rule $f^\epsilon_{\boldsymbol{\cdot}}$ as follows: given $x_{\leq t}\in \Xcal^t$ and $y_{<t}\in \Ycal^{t-1}$,
\begin{equation*}
    f^\epsilon_t(x_{\leq t},y_{<t},x_t) = y^{f^\Nbb_t(x_{\leq t},h_{\epsilon}(y_{<t}),x_t)}.
\end{equation*}
By construction, at each step if the prediction on $h_\epsilon$ is successful, the loss of $f_t^\epsilon$ is at most $\epsilon$. If the prediction of $h_\epsilon$ fails, we can upper bound the loss by $\bar \ell$:
\begin{equation*}
    \ell(f^\epsilon_t(x_{\leq t},y_{<t},x_t),y_t)\leq \epsilon + \bar \ell \cdot  \ell_{01}(f^\Nbb_t(x_{\leq t},h_\epsilon(y)_{<t},x_t),h_\epsilon(y)_t)
\end{equation*}
where $h_\epsilon(y):=(h_\epsilon(y_t))_{t\geq 1}$. Therefore, for any target measurable function $f^*:\Xcal\rightarrow \Ycal$, we obtain $\Lcal_\Xbb^{(\Ycal,\ell)}(f^\epsilon_{\boldsymbol{\cdot}},f^*;T)\leq \epsilon + \overline{\ell}\Lcal_{\Xbb}^{(\Nbb,\ell_{01})}(f^\epsilon_{\boldsymbol{\cdot}},h_\epsilon \circ f^*;T),$ where $\Lcal_{\Xbb}^{(\Nbb,\ell_{01})}(f^\epsilon_{\boldsymbol{\cdot}},h_\epsilon \circ f^*;T)\rightarrow 0 ~~(a.s.)$. Unfortunately, using the learning rule $f^\epsilon_{\boldsymbol{\cdot}}$ only ensures $\Lcal_\Xbb^{(\Ycal,\ell)}(f^\epsilon_{\boldsymbol{\cdot}},f^*)\leq \epsilon$ almost surely. Thus, the final learning rule will use the learning rules $f^{\epsilon_k}_{\boldsymbol{\cdot}}$ for a sequence of $\epsilon_k$ decreasing to $0$ e.g. $\epsilon_k=2^{-k}$. Intuitively, each learning rule $f^{\epsilon_k}_{\boldsymbol{\cdot}}$ with prediction $y^i$ effectively predicts that the output $y_t$ belongs to the set $\Bcal_i^{\epsilon_k}:= B_{\ell}(y^i,\epsilon_k)\setminus \bigcup_{1\leq j<i} B_{\ell}(y^j,\epsilon_k)$ where we used the notation $B_\ell(y,\epsilon) = \{y'\in \Ycal, \ell(y,y')<\epsilon\}$ for the ``ball" induced by the loss $\ell$. We now consider the learning rule on $(\Ycal,\ell)$ denoted $\hat f_{\boldsymbol{\cdot}}^{(\Ycal,\ell)}$ which successively checks consistency of these set predictions $f^{\epsilon_1}_{\boldsymbol{\cdot}},f^{\epsilon_2}_{\boldsymbol{\cdot}}$ etc. and outputs a point $\hat y\in\Ycal$ close to the consistent intersection of these sets. Formally, 
\begin{equation*}
    \hat f_t^{(\Ycal,\ell)}(x_{\leq t},y_{<t},x_t) = f^{\epsilon_{\hat p}}_t(x_{\leq t},y_{<t},x_t) \text{ for }\hat p=\max\left \{1\leq  p\leq t, \bigcap_{1\leq k\leq p} \Bcal_{f^{\epsilon_k}_t(x_{\leq t},y_{<t},x_t)}^{\epsilon_k} \neq \emptyset \right\}.
\end{equation*}
In this definition, the upper bound $\hat p\leq t$ is put for simplicity only to ensure that there is a finite maximum. We can now show that this learning rule is universally consistent.

Let $k\geq 1$. Note that if the predictions at step $t\geq k$ of $f^{\epsilon_l}_t$ were correct for all $1\leq l\leq k$, then the true output $y_t$ belongs to each set prediction $y_t\in \bigcap_{1\leq l\leq k} \Bcal_{f^{\epsilon_l}_t(x_{\leq t},y_{<t},x_t)}^{\epsilon_l}$, thus $\hat p\geq k$. Now let any $\bar y\in \bigcap_{1\leq l\leq \hat p} \Bcal_{f^{\epsilon_l}_t(x_{\leq t},y_{<t},x_t)}^{\epsilon_l}$, by relaxed triangle inequality we would have
\begin{equation*}
    \ell(\hat f_t^{(\Ycal,\ell)}(x_{\leq t},y_{<t},x_t),y_t)\leq c_\ell ( \ell(\hat f_t^{(\Ycal,\ell)}(x_{\leq t},y_{<t},x_t),\bar y) + \ell(y_t,\bar y)) \leq c_\ell(\epsilon_{\hat p} + \epsilon_k) \leq 2c_\ell \epsilon_k.
\end{equation*}
Hence,
\begin{equation*}
    \ell(\hat f_t^{(\Ycal,\ell)}(x_{\leq t},y_{<t},x_t),y_t) \leq 2c_\ell \epsilon_k  + \bar \ell \cdot \sum_{l=1}^k \ell_{01}(f^\Nbb_t(x_{\leq t},h_{\epsilon_k}(y)_{<t},x_t),h_{\epsilon_k}(y)_t),
\end{equation*}
and for any measurable function $f^*:\Xcal\to\Ycal$, we have $   \Lcal_\Xbb^{(\Ycal,\ell)}(f^{(\Ycal,\ell)}_{\boldsymbol{\cdot}},f^*)\leq 2c_\ell \epsilon_k\; (a.s.).$ By union bound, almost surely this holds for any $k\geq 1$ simultaneously. Therefore, almost surely $\Lcal_\Xbb^{(\Ycal,\ell)}(f^{(\Ycal,\ell)}_{\boldsymbol{\cdot}},f^*)=0$ and the learning rule $f^{(\Ycal,\ell)}_{\boldsymbol{\cdot}}$ is universally consistent.
\end{proof}
The results of \cite{hanneke2021learning} offer more details that are not required in the rest of the paper but can be found in Appendix \ref{appendix:alternative_reduction}.

\subsection{Open problem 3}
For any near-metric space $(\Ycal,\ell)$, the inclusions $\suol_{(\Nbb,\ell_{01})}\subset\suol_{(\Ycal,\ell)}\subset\suol_{(\{0,1\},\ell_{01})}$ given in Proposition \ref{prop:binary_izi} and Theorem \ref{thm:reduction_countable} do not answer whether $\suol_{(\Ycal,\ell_{01})}$ is invariant to the setup when the loss is bounded. The remaining question is whether $\textnormal\suol_{(\{0,1\},\ell_{01})} \subset \textnormal\suol_{(\Nbb,\ell_{01})}$ holds or not. We answer positively to this question in the next section, thereby providing a solution to the following open problem. 

\paragraph{Open Problem 3 (\cite{hanneke2021learning}):} \textit{Is the set $\textnormal\suol$ invariant to the specification of $(\Ycal,\ell)$, subject to $(\Ycal,\ell)$ being separable with $0<\bar\ell<\infty$?}\\

\noindent
\paragraph{Remarks on Open Problem 3.} In words, the open problem asks whether any universal learning task is achievable whenever universal binary classification is possible. In order to answer affirmatively it would suffice to show that the countable classification setting can be reduced to the binary classification setting. Given a process $\Xbb$ admitting universal learning for binary classification and a countable classification task $f^*:\Xcal\to\Nbb$, a natural idea would be to solve separately each of the binary classification tasks $f^{*,i}(\cdot) = \1(f^*(\cdot)=i)$ for $i\in \Nbb$ and to merge the results together. This proof technique works when $f^*$ takes only a finite number of values, giving rise to the following lemma. Its proof can be found in Appendix \ref{ap:lemma_finite_classification}.
\begin{lemma}
\label{lemma:finite_classification}
For any $k\geq 2$, $\textnormal\suol_{([k],\ell_{01})}= \textnormal\suol_{(\{0,1\},\ell_{01})}$.
\end{lemma}

Unfortunately, the proof technique used to show that finitely-many classification reduces to binary classification does not extend to countably-many classification. Indeed, the rate of convergence of the average loss on the tasks $f^{*,i}(\cdot)=\1(f^*(\cdot)=i)$ is not uniform across $i\in\Nbb$. Thus, although we can wait for the convergence of a fixed number of these predictors---say the predictions for functions $f^{*,1},\ldots,f^{*,k}$---we do not have any guarantee on the average losses of the predictions for the next functions $f^{*,i}$ for $i>k$. Essentially, we can only guarantee low average loss for a finite number of predictors which use binary classification.

Our proof differs substantially from this approach by considering instead a very large set of predictors---uncountably many. However, we introduce a probability distribution on these predictors, which allows to have guarantees on the average loss for the predictor with high probability on both the stochastic process $\Xbb$ and the predictor. More precisely, instead of learning the individual label $i$, $f^{*,i}(\cdot)=\1(f^*(\cdot)=i)$, we use predictors of sets of labels $\sigma\in \Pcal(\Nbb)$ as follows: $f_\sigma^*(\cdot) = \1(f^*(\cdot)\in \sigma)$. We can now introduce a uniform distribution for the variable $\sigma$ and test the hypothesis $f^*(x_t)=i$ by analysing the probability (in $\sigma$) of the prediction for $f_\sigma^*$ to be consistent with this hypothesis i.e. $f_\sigma^*(x_t)=1$ if $f^*(x_t)\in \sigma$ and  $f_\sigma^*(x_t)=0$ if $f^*(x_t)\notin \sigma$. Intuitively, for the right hypothesis $i^* = y_t$, this probability will be close to $1$, while for a wrong hypothesis $i^* \neq y_t$ consistency either results from errors in the predictors, or that both $i,i^*\in \sigma$ or both $i,i^*\notin \sigma$ which happens with probability $1/2$. This discrepancy in probability will allow to discriminate which is the true hypothesis with sublinear number of mistakes.

\section{Reduction of countable classification to binary classification}\label{sec: result reduction binary}

We now present the proof of the main technical result of this paper.
\begin{theorem}
\label{thm:binary_countable}
$\normalfont \suol_{(\{0,1\},\ell_{01})} \subset \suol_{(\Nbb,\ell_{01})}$.
\end{theorem}

\begin{proof}
%We start by a recall of the inclusion $\suol_{\Nbb} \subset \suol_{\{0,1\}}$. Let $\Xbb\in \suol_{\Nbb}$ and $f_{\boldsymbol{\cdot}}:=(f_t)_{t=1}^\infty$ a universally consistent learning rule for the input sequence $\Xbb$ and the setting $(\Nbb, \ell_{01})$. By definition this learning rule is consistent for any function $f^*:\Xcal \to \Nbb$, \emph{a fortiori} for any function $f^*:\Xcal \to \{0,1\}$.
%Let us now prove the inclusion $\suol_{\{0,1\}} \subset \suol_{\Nbb}$.
Suppose you have a process $\Xbb\in \suol_{\{0,1\}}$. We want to show that there exists some universal learner for the input process $\Xbb$ and the setting $(\Nbb, \ell_{01})$. Denote by $f_{\boldsymbol{\cdot}}:=\{f_t\}_{t=1}^\infty$ the universal learner in the binary classification setting $(\{0,1\},\ell_{01})$ for sequence $\Xbb$ and by $f^*:\Xcal\to\Nbb$ the unknown function to learn. For some subsets of outputs $S\subset \Nbb$ we will consider learning the binary valued function $f^*_S(\cdot) = \1 (f^* (\cdot)\in S)$.

Specifically, we introduce a random set $\sigma \subset \Nbb$ defined on the product topology of independent Bernoullis. Let $(B_j)_{j\geq 0}$ a sequence of i.i.d. Bernouilli $\Bcal(1/2)$, we define $\sigma = \{j\geq 1:B_j=1\}$. Based on learning the functions $f^*_\sigma$ we now define a statistical test which we will use to define a learning rule for the countable classification. Precisely, given a time $t\geq 0$, define for all $i\in \Nbb$,
\begin{equation*}
    p^t(x_{<t},y_{<t},x_t;i) := \frac{\Pbb_{\sigma}\left[f_t(x_{<t},\1(y\in \sigma)_{<t},x_t) = 1 \mid  i\in \sigma \right] + \Pbb_{\sigma}\left[f_t(x_{<t},\1(y\in \sigma)_{<t},x_t) = 0 \mid  i\notin\sigma \right]}{2}.
\end{equation*}
where we slightly abuse notations and write $\1(y\in \sigma)$ to denote $(\1(y_t\in \sigma))_{t\geq 1}$. Intuitively, $p^t(\Xbb_{<t},\Ybb_{<t},X_t;i)$ gives the proportion of subsets $\sigma$ for which the hypothesis $f^*(X_t)=i$ would be consistent with the prediction on the model trained to predict $f^*_\sigma(X_t)$. We first note that although the definition of $p^t(x_{<t},y_{<t},x_t;i)$ involves computing expectations over the product measure for $\sigma$, its computation can be made practical by considering the values of $B_j$ for observed values $j$, i.e. $j\in  \{y_{t'}:t'<t\}:=\Ycal$. Indeed, we can conveniently write $p^t(x_{<t},y_{<t},x_t;i)$ as
\begin{multline*}
    p^t(x_{<t},y_{<t},x_t;i) = \frac{1}{2^{|\Ycal|}}\sum_{(b_j)_{j\in\Ycal}\in\{0,1\}^{|\Ycal|}} \Pbb[f_t(x_{<t},(b_{y_{t'}})_{t'<t},x_t) = 1]\1(b_i=1) \\
    + \Pbb[f_t(x_{<t},(b_{y_{t'}})_{t'<t},x_t) = 0]\1(b_i=0),
\end{multline*}
where the probability is taken on the possible randomness of the learning rule only. As a result, the function $p^t(\cdot,\cdot,\cdot;\cdot)$ can be practically computed and is also measurable.

Note that if the learning rule $f_{\boldsymbol{\cdot}}$ had no errors we would have a simple discrimination as follows
\begin{equation*}
    \frac{\Pbb_{\sigma}\left[f^*_\sigma(X_t) = 1 \mid  i\in \sigma \right] + \Pbb_{\sigma}\left[f^*_\sigma(X_t) = 0 \mid  i\notin\sigma \right]}{2} = \begin{cases}
    1 &\text{if } f^*(X_t)=i,\\
    1/2 &\text{otherwise}.
    \end{cases}
\end{equation*}
We are now ready to define a learning rule $\hat f:=\{\hat f_t\}_{t=1}^\infty$ for countable classification as follows
\begin{equation*}
    \hat f_t(x_{<t},y_{<t},x_t):= \begin{cases}
    \displaystyle{ \min_{i\in \Nbb}} \left\{i:\; p^t(x_{<t},y_{<t},x_t;i)>\frac{3}{4} \right\} & \text{if }\exists i\in \Nbb,\; p^t(x_{<t},y_{<t},x_t;i)>\frac{3}{4} \\
    0 & \text{otherwise}.
    \end{cases}
\end{equation*}
This is a valid measurable learning rule as a result of the measurability of $p^t(\cdot,\cdot,\cdot;\cdot)$ for all $t\geq 1$.
We now show that the learning rule $\hat f$ is universally consistent. By hypothesis of binary classification universal consistency, for any subset $S\in  \Pcal(\Nbb)$, we have $\Pbb_\Xbb[\Lcal_\Xbb (f_{\boldsymbol{\cdot}},f^*_S;T) \xrightarrow[T ]{}  0]=1.$ Because this result is true for any subset $S$, we get
\begin{equation*}
    \Pbb_{\Xbb,\sigma}\left[\Lcal_\Xbb (f_{\boldsymbol{\cdot}},f^*_\sigma;T) \xrightarrow[T\to\infty ]{}  0\right]=1
\end{equation*}
where the randomness is taken on both $\Xbb$ and $\sigma$ -- and potentially the learning process $f_{\boldsymbol{\cdot}}$. Therefore, we have 
\begin{equation*}
    \Pbb_{\sigma}\left[\Lcal_\Xbb (f_{\boldsymbol{\cdot}},f^*_\sigma;T) \xrightarrow[T\to\infty ]{}  0\right]=1,\quad \text{a.s. in }\Xbb
\end{equation*}
Denote by $\Ecal$ this event of probability $1$. We will show that on this event, the learning rule is consistent. We now fix an input trajectory $\Xbb$ falling in $\Ecal$ which we denote by $x=(x_t)_{t=0}^\infty$ to make clear that there is no randomness on the trajectory anymore -- one can think of a deterministic process. We additionally denote $y=(y_t)_{t=0}^\infty:=(f^*(x_t))_{t=0}^\infty$ for simplicity.

By construction, for any $\epsilon>0$ we have 
\begin{equation*}
    \Pbb_{\sigma}\left[\Lcal_x (f_{\boldsymbol{\cdot}},f_\sigma^*;t) \leq \epsilon,\quad \forall t\geq T\right] \xrightarrow[T \to \infty]{} 1
\end{equation*}
We can then define for any $\epsilon$ a time $T_{\epsilon}\geq 0$ such that

\begin{equation*}
    \Pbb_{\sigma} \left[\Lcal_x (f_{\boldsymbol{\cdot}},f^*_\sigma;t) \leq \epsilon,\quad \forall t\geq T_{\epsilon}\right] \geq \frac{7}{8}.
\end{equation*}
We define the event $\Acal_{\epsilon}=\{\Lcal_x (f_{\boldsymbol{\cdot}},f^*_\sigma;t) \leq \epsilon,\quad \forall t\geq T_{\epsilon}\}$. An important remark is that both $T_\epsilon$ and the event $\Acal_\epsilon$ are dependent on the specific trajectory $x$: the learning rate of our rule depends on the realization of the input trajectory. We will show that from time $T_{\epsilon}$, the error rate of $\hat f$ is at most $8\epsilon$. Let $t\geq 0$ and $i^*_t=f^*(x_t)$ be the true (random) value that we want to predict. We have for the true value $i_t^*$,
\begin{align*}
    p^t(x_{<t},y_{<t},x_t;i^*_t) & = 1-\frac{\Pbb_{\sigma}\left[f_t(x_{<t},f^*_\sigma(x_{<t}),x_t) = 0 \mid  i_t^*\in \sigma \right] + \Pbb_{\sigma}\left[f_t(x_{<t},f^*_\sigma(x_{<t}),x_t) = 1 \mid  i_t^*\notin\sigma \right]}{2}\\
    &\geq 1- \Pbb_\sigma[\bar\Acal_{\epsilon}] - \Ebb_{\sigma}\left[\left(\1_{f_t(x_{<t},f^*_\sigma(x_{<t}),x_t) = 0 ,  i_t^*\in \sigma} + \1_{f_t(x_{<t},f^*_\sigma(x_{<t}),x_t) = 1,i_t^*\notin\sigma}\right)  \1_{\Acal_{\epsilon}} \right]\\
    &\geq 1- \frac{1}{8} - \Ebb_{\sigma}\left[\ell(f_t(x_{< t},f^*_\sigma(x_{<t}), x_t), f^*_\sigma(x_t))\1_{\Acal_{\epsilon}}\right]
\end{align*}
However, for any $i\neq i^*_t$,
\begin{align*}
    p^t(x_{<t},y_{<t},x_t;i) & = \frac{\Pbb_{\sigma}\left[f_t(x_{<t},f^*_\sigma(x_{<t}),x_t) = 1 \mid  i\in \sigma \right] + \Pbb_{\sigma}\left[f_t(x_{<t},f^*_\sigma(x_{<t}),x_t) = 0 \mid  i\notin\sigma \right]}{2}\\
    &\leq \frac{1}{2} + \Ebb_\sigma \left[\1_{f_t(x_{<t},f^*_\sigma(x_{<t}),x_t) = 1,i\in \sigma,i^*_t\notin \sigma } + \1_{f_t(x_{<t},f^*_\sigma(x_{<t}),x_t) = 0, i\notin\sigma,i^*_t\in \sigma}\right]\\
    &\leq \frac{1}{2} + \Pbb_\sigma[\bar\Acal_{\epsilon }] + \Ebb_{\sigma}\left[(\1_{f_t(x_{<t},f^*_\sigma(x_{<t}),x_t) = 1 ,  i\in \sigma,i^*_t\notin\sigma} + \1_{f_t(x_{<t},f^*_\sigma(x_{<t}),x_t) = 0,i\notin\sigma,i^*_t\in\sigma})\1_{\Acal_{\epsilon }}  \right]\\
    &\leq \frac{1}{2}+ \frac{1}{8} + \Ebb_{\sigma}\left[\ell(f_t(x_{< t},f^*_\sigma(x_{<t}), x_t), f^*_\sigma(x_t))\1_{\Acal_{\epsilon }}\right]
\end{align*}
Note that the term $e_t := \Ebb_{\sigma}\left[\ell(f_t(x_{< t},f^*_\sigma(x_{<t}), x_t), f^*_\sigma(x_t))\1_{\Acal_{\epsilon }}\right]$ is a simple scalar. Therefore, by the previous estimates on $p^t$, whenever $e_t < \frac{1}{8}$, the learning rule classifies the new input point correctly: $\1_{\hat f_t(x_{<t},y_{<t},x_t)\neq i_t^*} \leq \1_{e_t\geq\frac{1}{8}}.$ We will now show that the bad event $e_t \geq \frac{1}{8}$ only happens with sublinear rate in $t$. By construction, in $\Acal_{\epsilon }$, for any $t\geq T_{\epsilon }$,
 \begin{equation*}
     \frac{1}{t}\sum_{u=1}^{t}\ell(f_t(x_{< t},f^*_\sigma(x_{<t}), x_t), f^*_\sigma(x_t)) \leq \epsilon.
 \end{equation*}
Therefore, for any $t\geq T_{\epsilon }$, we have 
\begin{equation*}
    \frac{1}{t}\sum_{u=1}^T e_u= \frac{1}{t}\sum_{u=1}^{t}\Ebb_{\sigma}\left[\ell(f_t(x_{< t},f^*_\sigma(x_{<t}), x_t), f^*_\sigma(x_t)) \1_{\Acal_{\epsilon }}\right] \leq \epsilon.
\end{equation*}
The loss of our learning rule on trajectory $x$ now satisfies for all $t\geq T_{\epsilon }$,
\begin{equation*}
    \Lcal_x (\hat f_{\boldsymbol{\cdot}},f^*;t) = \frac{1}{t}\sum_{u=1}^t \1_{\hat f_u(x_{<u},y_{<u},x_u)\neq i_u^*} \leq \frac{1}{t}\sum_{u=1}^t  \1_{e_u\geq\frac{1}{8}} \leq  \frac{8}{t} \sum_{u=1}^t e_u \leq 8 \epsilon.
\end{equation*}
Thus, $\Lcal_x(\hat f_{\boldsymbol{\cdot}},f^*)\leq 8\epsilon.$ Taking $\epsilon>0$ arbitrarily small shows that $\Lcal_x(\hat f_{\boldsymbol{\cdot}},f^*)=0$ and hence, the learning rule is consistent on trajectory $x$. Therefore, $\hat f_{\boldsymbol{\cdot}}$ is consistent on the event $\Ecal$ for the input sequence $\Xbb$, which has probability $1$. To summarize, $\Lcal_\Xbb(\hat f_{\boldsymbol{\cdot}},f^*)=0\; ( a.s.)$ for any measurable function $f^*$, showing that $\hat f_{\boldsymbol{\cdot}}$ is universally consistent and thus $\Xbb\in \suol_\Nbb$. This ends the proof of the theorem.
\end{proof}

Together with Theorem \ref{thm:reduction_countable}, this theorem ends the proof of the main result Theorem \ref{thm:invariance_suol}. Theorem \ref{thm:invariance_optimistic} is also a direct consequence from the proof of Theorem \ref{thm:binary_countable}, Theorem \ref{thm:reduction_countable} and Proposition \ref{prop:binary_izi} since the learning rules were all constructed independently from the stochastic process $\Xbb$. The complete proof is given in Appendix \ref{appendix:invariance_optimistic} .

\subsection{Learning rules preserved by the reduction}
Though its definition is little abstruse, the countable classification learning rule that is derived from the proof of Theorem \ref{thm:reduction_countable} leaves many learning rules unchanged. In particular, the following proposition shows that learning rules based on a representant which depends only on the historical input sequence e.g. nearest neighbor rule, are transported by our construction.
\begin{proposition}\label{prop:rule_preserved}
Let $\{f_t\}_{t=1}^\infty$ be a learning rule defined by %for $x_{<t}\in \Xcal^{t-1},y_{<t}\in \Ycal^{t-1},x_t\in \Xcal$, 
representant function $\phi(t)\in \{1,...,t-1\}$ which at step $t$ only depends on $(x_1,..,x_t)$ as follows,
\begin{equation*}
    f_t(x_{<t},y_{<t},x_t) = y_{\phi(t)}.
\end{equation*}
Note that this learning rule can be defined for any output setting $(\Ycal,\ell)$. If $\{f_t\}_{t=1}^\infty$ is universally consistent on a process $\Xbb$ for binary classification, it is also universally consistent on $\Xbb$ for any separable near-metric setting $(\Ycal,\ell)$ with bounded loss.
\end{proposition}
\label{prop:preserved_learning_rules}
\begin{proof}
We first show that the learning rule $f_{\boldsymbol{\cdot}} = \{f_t\}_{t=1}^\infty$ is transported by our construction in Theorem \ref{thm:binary_countable} for classification with countable number of classes. In the rest of the proof, we will denote by $\phi(\cdot)$ the representant function of $f_{\boldsymbol{\cdot}}$. With
\begin{multline*}
    p^t(x_{<t},y_{<t},x_t;i) :=\frac{1}{2}\left( \Pbb_{\sigma}\left[f_t(x_{<t},\1(y\in \sigma)_{<t},x_t) = 1 \mid  i\in \sigma \right] \right.\\+ \left.\Pbb_{\sigma}\left[f_t(x_{<t},\1(y\in \sigma)_{<t},x_t) = 0 \mid  i\notin\sigma \right]\right),
\end{multline*}
we define our learning rule $f^\Nbb_{\boldsymbol{\cdot}}:=\{ f_t^\Nbb\}_{t=1}^\infty$ for countably-many classification as in Theorem \ref{thm:binary_countable}:
\begin{equation*}
    f_t^\Nbb(x_{<t},y_{<t},x_t):= \begin{cases}
    \displaystyle{ \min_{i\in \Nbb}} \left\{i,\; p^t(x_{<t},y_{<t},x_t;i)>\frac{3}{4} \right\} & \text{if }\exists i\in \Nbb,\; p^t(x_{<t},y_{<t},x_t;i)>\frac{3}{4} \\
    0 & \text{otherwise}.
    \end{cases}
\end{equation*}
We now show that $f^\Nbb_{\boldsymbol{\cdot}}$ is in fact defined with a similar representant function. Indeed, 
\begin{align*}
    p^t(x_{<t},y_{<t},x_t;i) &= \frac{\Pbb_{\sigma}\left[\1(y_{\phi(t)}\in \sigma) = 1 \mid  i\in \sigma \right] + \Pbb_{\sigma}\left[\1(y_{\phi(t)}\in \sigma) = 0 \mid  i\notin\sigma \right]}{2} \\
    &=  \begin{cases}
    1 \text{ if } y_{\phi(t)} = i\\
    \frac{1}{2} \text{ if } y_{\phi(t)} \neq i
    \end{cases} 
\end{align*}
Therefore, we obtain $f_t^\Nbb(x_{<t},y_{<t},x_t) = y_{\phi(t)}$, which shows that $f^\Nbb_{\boldsymbol{\cdot}} = f_{\boldsymbol{\cdot}}$ i.e. that the learning rule $f_{\boldsymbol{\cdot}}$ is transported by the construction.

We now fix a separable near-metric space $(\Ycal,\ell)$ and a process $\Xbb$ such that $f_{\boldsymbol{\cdot}}$ is universally consistent for binary classification. By the above arguments, Theorem \ref{thm:binary_countable} shows that $f_{\boldsymbol{\cdot}}$ is also universally consistent for countable classification. We now aim to show that $f_{\boldsymbol{\cdot}}$ on $(\Ycal,\ell)$ is universally consistent on $\Xbb$. Let $f^*$ be a measurable target function and $\epsilon>0$. We take a sequence $(y^i)_{i\geq 1}$ dense on $\Ycal$ with respect to $\ell$ and construct the function $h(y) = \inf\{i\geq 1,\ell(y^i,y)<\epsilon\}$. Then, $y^{f_t(x_{< t},h_k(y_{<t}),x_t)} = y^{h(y_{\phi(t)})}$. Hence, if $f_t(x_{< t},h(y_{<t}),x_t)= h(y_t)$ we obtain $y^{h(y_{\phi(t)})} = y^{h(y_t)}$. Therefore, we can write
\begin{align*}
    \ell(y_{\phi(t)},y_t)&\leq \bar \ell  \cdot \1_{f_t(x_{< t},h(y_{<t}),x_t)\neq h(y_t)} + \ell(y_{\phi(t)},y_t)\1_{f_t(x_{< t},h(y_{<t}),x_t)= h(y_t)}\\
    &\leq \bar  \ell \cdot \1_{f_t(x_{< t},h(y_{<t}),x_t)\neq h(y_t)}  + c_\ell ( \ell(y_{\phi(t)}, y^{h(y_{\phi(t)})})+ \ell(y^{h(y_t)},y_t))\\
    &\leq \bar \ell\cdot \ell_{01}(f_t(x_{< t},h(y_{<t}),x_t), h(y_t)) + 2c_\ell\epsilon.
\end{align*}
This yields $\Lcal_\Xbb(f_{\boldsymbol{\cdot}},f^*;T)\leq \bar \ell \Lcal_\Xbb(f_{\boldsymbol{\cdot}},h\circ f^*;T) + 2c_\ell \epsilon$. Because $f_{\boldsymbol{\cdot}}$ is universally consistent for the setting $(\Nbb,\ell_{01})$, it is in particular consistent for target function $h\circ f^*:\Xcal\to \Nbb$. Therefore, $\limsup_T \Lcal_\Xbb(f_{\boldsymbol{\cdot}},f^*;T)\leq 2c_\ell\epsilon,\; (a.s.)$. This is valid for $\epsilon_k=2^{-k}$ for all $k\geq 1$. Therefore, by union bound, $\Lcal_\Xbb(f_{\boldsymbol{\cdot}},f^*;T)\to 0,\;(a.s.)$, which ends the proof that $f_{\boldsymbol{\cdot}}$ is universally consistent on $\Xbb$ for the setting $(\Ycal,\ell)$.
\end{proof}

\section{Properties of the 1-Nearest Neighbour learning rule}\label{sec:nearest_neighbour}

In this section, we will study some interesting properties of the simple nearest neighbour learning rule in the context of strong universal online learning. Formally, we can define $\nn = \{\nn_{t}\}_{t=1}^\infty$ as follows: for $t>1$,
\begin{equation*}
    \nn_t((x_i)_{i<t},(y_i)_{i<t},x_t) = 
    y_{\phi(t)} \quad \text{where } \phi(t) = \argmin_{1\leq i<t} \rho(x_t,x_i).
\end{equation*}
Ties can be broken arbitrarily, for simplicity we split ties in favor of the most ancient closest input point. We will refer to $x_{\phi(t)}$ as the representant of $x_t$ for the nearest neighbor rule. Proposition \ref{prop:preserved_learning_rules} shows that if nearest neighbor is universally consistent for process $\Xbb$ in the binary classification setting, it is also universally consistent for any bounded separable near-metric setting $(\Ycal,\ell)$. As an immediate consequence of this result, all processes that are i.i.d admit nearest neighbour as an universally consistent learning rule for any near-metric setting. This comes from the universal consistency of nearest neighbour on such processes for binary classification \cite{devroye2013probabilistic}.

This reduction motivates the analysis of the consistency of nearest neighbour for binary classification. In the rest of this section, we will focus on the specific case $\Xcal = \Rbb$ with classical Euclidian distance $\rho$ as metric. The results presented here can be extended to the $d$-dimensional euclidean space $\Rbb^d$. We will show that if the input stream $\Xbb$ is in $\suol$, the nearest neighbour learning rule is at least able to learn functions that represent a finite union of intervals which are in some sense ``simple" functions. 

\begin{theorem}
For any process $\Xbb\in \suol$, the nearest neighbour learning rule is consistent for any finite union of intervals $A=\bigcup_{k=1}^n I_k$ for any arbitrary $n\geq 1$, i.e., for $f^* = \1(\cdot \in A)$ we have $ \Lcal_\Xbb(\nn,f^*,T)\rightarrow0 ~(a.s.).$ 
\end{theorem}
\noindent The proof of this result comes in two steps. First we show that the collections of set that are consistent with the nearest neighbour learning rule is closed by complement and finite union. Second, we show that this collection contains the intervals. Note that in order to prove universal consistency of the nearest neighbour learning rule, we would need to prove that this collection is closed under \emph{countable} union. This is unfortunately beyond the results of this paper.

To build some intuition on the significance of the result, we provide a simple process (deterministic, yet not in $\suol$) for which nearest neighbor fails on an interval. Let $\Xcal = [-1,1]$, $f^*=\1_{[0,1]}$ and $X_t = (-\frac{1}{3})^t$. Then, the nearest neighbor of $X_t$ is $X_{t-1}$ for all $t\geq 1$, inducing an error at each step. On the other hand, $\suol$ processes do not have this behavior.\\

\begin{proof} We fix a stochastic process $\Xbb \in \suol$. Recall that as a consequence $\Xbb$ satisfies condition $\smv$ (Condition 2 in \cite{hanneke2021learning}). This condition states that $\Xbb$ can only makes a sublinear number of visits of different regions of any measurable partition of $\Xcal$. The condition is formally stated as follows.\\\\
\noindent{\bf Condition \smv}~~ {\it The stochastic process $\Xbb$ satisfies condition $\smv$ i.if for every disjoint sequence $\{A_k\}_{k=1}^{\infty}$ in $\Bcal$ with $\bigcup_{k=1}^{\infty}A_k=\Xcal$ (i.e., every countable measurable partition),}
\begin{equation*}
    \#\{k\in \mathbb N: A_k\cap \mathbb{X}_{<T} \neq \emptyset \} = o(T) \quad (a.s).
\end{equation*}

We define $\Fcal_{\Xbb}$ the collection of measurable sets $A\in \Bcal$ for which the nearest neighbour learning rule is consistent on the associated indicator function $\1(\cdot \in A)$. Formally,
\begin{equation*}
    \Fcal_{\Xbb} = \{ A \in \Bcal \mid \Lcal_{\Xbb}(\nn,\1(\cdot \in A);T)\rightarrow 0,\; (a.s.)\}.
\end{equation*}
Note that $\Fcal_\Xbb$ is stable by complement because the choice of representant in the nearest neighbor rule is independent of the target function: for any measurable set $A$ we have $\Lcal_{\Xbb}(\nn,\1_{\cdot \in A};T) = \Lcal_{\Xbb}(\nn,1-\1_{\cdot \in A};T) = \Lcal_{\Xbb}(\nn,\1_{\cdot \in A^c};T)$.

We now show that $\Fcal_{\Xbb}$ is stable by finite union. Let $A_i\in \Fcal_\Xbb$ for $i=1,\ldots,k$. For simplicity we denote $A:=\bigcup_{i=1}^k A_i$. Again, because the choice of representant is independent from the target function, if nearest neighbor makes a mistake at time $t$ for target function $A$, it makes at also a mistake for at least one of the functions $A_i$, $1\leq i\leq k$.
\begin{align*}
    \Lcal_{\Xbb}(\nn,\1(\cdot \in A);T) &= \sum_{t=1}^T \1_{x_t\in A}\1_{x_{\phi(t)}\notin A} + \1_{x_t\notin A}\1_{x_{\phi(t)}\in A}\\
    &\leq \sum_{t=1}^T \sum_{i=1}^k 1_{x_t\in A_i}\1_{x_{\phi(t)}\notin A} + \1_{x_t\notin A}\1_{x_{\phi(t)}\in A_i}\\
    &=\sum_{i=1}^k \Lcal_{\Xbb}(\nn,\1(\cdot \in A_i);T).
\end{align*}
Since $A_i\in \Fcal_\Xbb$ we have $\Lcal_{\Xbb}(\nn,\1(\cdot \in A_i);T) \rightarrow 0, (a.s).$ Therefore, we obtain directly $\Lcal_{\Xbb}(\nn,\1(\cdot \in A);T) \rightarrow 0, (a.s)$ i.e. $A\in \Fcal_\Xbb$.

We now show that $\Fcal_\Xbb$ contains all intervals of the form $(-\infty,a)$ and $(-\infty,a]$ for $a\in \Rbb$. Let $f^*=\1_{\cdot\in (-\infty,a)}$ or $f^*=\1_{\cdot\in (-\infty,a]}$ and consider the following countable partition
\begin{equation*}
    \Pcal:\quad \{a\}\cup \bigcup_{i\in \Zbb}[a+2^i,a+2^{i+1}) \cup \bigcup_{i\in \Zbb}(a-2^{i+1},a-2^i].
\end{equation*}
For any $t\geq 1$, let $P_t\in \Pcal$ the set of the partition in which $x_t$ falls. Observe that by construction, if there exists $u<t$ such that $x_u\in P_t$, then nearest neighbor classifies $x_t$ correctly. Indeed, assuming that $x_t>a$, we can write and $P_t  = (a+2^i,a+2^{i+1}]$. Then we have  $|x_{\phi(t)}-x_t|\leq |x_u-x_t|$. Therefore, $x_{\phi(t)}\geq x_t - |x_u-x_t| > a+2^i - 2^i = a.$ Therefore, $y_{\phi(t)} = f^*(x_t)$. The case $x_t<a$ is symmetric, and the case $x_t=a$ is immediate. Thus, if nearest neighbor makes a mistake at time $t$, the input $x_t$ visited a new set of the partition:
\begin{equation*}
    \Lcal_{\Xbb}(\nn,f^*; T) \leq \frac{1}{T}|\{k\in \mathbb N: A_k\cap \mathbb{X}_{<T} \neq \emptyset \}|.
\end{equation*}
Because $\Xbb\in \smv$, we can apply the property to the countable partition $\Pcal$ and which yields $\Lcal_\Xbb(\nn,f^*;T)\to 0 ,\;(a.s.).$ This ends the proof of the theorem.
\end{proof}

\section{Conclusion}
\label{sec:conclusion}
We resolve an open problem of \cite{hanneke2021learning}. We present a novel reduction from a general (separable near-metric) setting to the binary classification setting in the context of universal online learning. This reduction shows that the stochastic processes admitting strong universal consistency for regression are exactly those admitting strong universal consistency for binary classification. Our proof technique is probabilistic but enjoys the property of transporting many natural learning rules such as nearest neighbour. We analyze this particular learning rule in the context of classification for finite union of intervals.

Though the nearest neighbour learning rule has already been extensively studied, there remain interesting questions related to its consistency. For a process $\Xbb$ in $\suol$, what is the class of functions $f^*$ for which nearest neighbour achieves strong consistency? In this paper, we showed in the context of binary classification that this class must contain finite unions of intervals, but the general class is possibly much larger.  Reciprocally, can we characterize the set of processes for which nearest neighbour is a strong universal online learning rule? 

On another note, this paper highlights the importance of the open problems formulated in \citep{hanneke2021open} for the binary classification setting --- the existence of an optimistically universal learning rule and the characterization of $\suol$. The present paper shows that any solution to these problems would transport from the binary classification setting to the general setting. The authors note that subsequently to this paper, the reduction presented in this work was applied by \cite{blanchard2022optimistically} to obtain optimistically universal learning rules for general metric value spaces.

\acks{Mo\"ise Blanchard is being partly funded by ONR grant N00014-18-1-2122.}

\bibliography{refs}

\begin{thebibliography}{22}
\providecommand{\natexlab}[1]{#1}
\providecommand{\url}[1]{\texttt{#1}}
\expandafter\ifx\csname urlstyle\endcsname\relax
  \providecommand{\doi}[1]{doi: #1}\else
  \providecommand{\doi}{doi: \begingroup \urlstyle{rm}\Url}\fi

\bibitem[Ben-David et~al.(2009)Ben-David, P{\'a}l, and
  Shalev-Shwartz]{ben2009agnostic}
Shai Ben-David, D{\'a}vid P{\'a}l, and Shai Shalev-Shwartz.
\newblock Agnostic online learning.
\newblock In \emph{COLT}, volume~3, page~1, 2009.

\bibitem[Blanchard(2022)]{blanchard2022optimistically}
Mo{\"\i}se Blanchard.
\newblock Universal online learning: an optimistically universal learning rule.
\newblock \emph{arXiv preprint arXiv:2201.05947}, 2022.

\bibitem[Blanchard et~al.(2022)Blanchard, Cosson, and
  Hanneke]{blanchard2022universal}
Mo\"ise Blanchard, Romain Cosson, and Steve Hanneke.
\newblock Universal online learning with unbounded losses: Memory is all you
  need.
\newblock In \emph{Proceedings of The 33rd International Conference on
  Algorithmic Learning Theory}, pages 107--127. PMLR, 2022.

\bibitem[Bousquet et~al.(2021)Bousquet, Hanneke, Moran, van Handel, and
  Yehudayoff]{bousquet2021theory}
Olivier Bousquet, Steve Hanneke, Shay Moran, Ramon van Handel, and Amir
  Yehudayoff.
\newblock A theory of universal learning.
\newblock In \emph{Proceedings of the 53rd Annual ACM SIGACT Symposium on
  Theory of Computing}, pages 532--541, 2021.

\bibitem[Cohen and Kontorovich(2022)]{cohen2022metric}
Dan~Tsir Cohen and Aryeh Kontorovich.
\newblock Metric-valued regression.
\newblock \emph{arXiv preprint arXiv:2202.03045}, 2022.

\bibitem[Devroye et~al.(1994)Devroye, Gyorfi, Krzyzak, and
  Lugosi]{devroye1994strong}
Luc Devroye, Laszlo Gyorfi, Adam Krzyzak, and G{\'a}bor Lugosi.
\newblock On the strong universal consistency of nearest neighbor regression
  function estimates.
\newblock \emph{The Annals of Statistics}, pages 1371--1385, 1994.

\bibitem[Devroye et~al.(2013)Devroye, Gy{\"o}rfi, and
  Lugosi]{devroye2013probabilistic}
Luc Devroye, L{\'a}szl{\'o} Gy{\"o}rfi, and G{\'a}bor Lugosi.
\newblock \emph{A probabilistic theory of pattern recognition}, volume~31.
\newblock Springer Science \& Business Media, 2013.

\bibitem[Gordon and Olshen(1978)]{gordon1978asymptotically}
Louis Gordon and Richard~A Olshen.
\newblock Asymptotically efficient solutions to the classification problem.
\newblock \emph{The Annals of Statistics}, pages 515--533, 1978.

\bibitem[Gy{\"o}fi and Lugosi(2002)]{gyofi2002strategies}
L{\'a}szl{\'o} Gy{\"o}fi and G{\'a}bor Lugosi.
\newblock Strategies for sequential prediction of stationary time series.
\newblock In \emph{Modeling uncertainty}, pages 225--248. Springer, 2002.

\bibitem[Gy\"{o}rfi and Ottucs\'ak(2007)]{gyorfi:07}
L\'asl\'o Gy\"{o}rfi and Gy\"orgy Ottucs\'ak.
\newblock Sequential prediction of unbounded stationary time series.
\newblock \emph{{IEEE} Transactions on Information Theory}, 53\penalty0
  (5):\penalty0 1866--1872, 2007.

\bibitem[Gy{\"o}rfi and Weiss(2021)]{gyorfi2021universal}
L{\'a}szl{\'o} Gy{\"o}rfi and Roi Weiss.
\newblock Universal consistency and rates of convergence of multiclass
  prototype algorithms in metric spaces.
\newblock \emph{Journal of Machine Learning Research}, 22\penalty0
  (151):\penalty0 1--25, 2021.

\bibitem[Hanneke(2021{\natexlab{a}})]{hanneke2021learning}
Steve Hanneke.
\newblock Learning whenever learning is possible: Universal learning under
  general stochastic processes.
\newblock \emph{Journal of Machine Learning Research}, 22\penalty0
  (130):\penalty0 1--116, 2021{\natexlab{a}}.

\bibitem[Hanneke(2021{\natexlab{b}})]{hanneke2021open}
Steve Hanneke.
\newblock Open problem: Is there an online learning algorithm that learns
  whenever online learning is possible?
\newblock In \emph{Conference on Learning Theory}, pages 4642--4646. PMLR,
  2021{\natexlab{b}}.

\bibitem[Hanneke et~al.(2021)Hanneke, Kontorovich, Sabato, and
  Weiss]{hanneke2021universal}
Steve Hanneke, Aryeh Kontorovich, Sivan Sabato, and Roi Weiss.
\newblock Universal bayes consistency in metric spaces.
\newblock \emph{The Annals of Statistics}, 49\penalty0 (4):\penalty0
  2129--2150, 2021.

\bibitem[Haussler et~al.(1994)Haussler, Littlestone, and
  Warmuth]{haussler1994predicting}
David Haussler, Nick Littlestone, and Manfred~K Warmuth.
\newblock Predicting $\{$0, 1$\}$-functions on randomly drawn points.
\newblock \emph{Information and Computation}, 115\penalty0 (2):\penalty0
  248--292, 1994.

\bibitem[Littlestone(1988)]{littlestone1988learning}
Nick Littlestone.
\newblock Learning quickly when irrelevant attributes abound: A new
  linear-threshold algorithm.
\newblock \emph{Machine learning}, 2\penalty0 (4):\penalty0 285--318, 1988.

\bibitem[Morvai et~al.(1996)Morvai, Yakowitz, and
  Gy{\"o}rfi]{morvai1996nonparametric}
Guszt{\'a}v Morvai, Sidney Yakowitz, and L{\'a}szl{\'o} Gy{\"o}rfi.
\newblock Nonparametric inference for ergodic, stationary time series.
\newblock \emph{The Annals of Statistics}, 24\penalty0 (1):\penalty0 370--379,
  1996.

\bibitem[M{\"u}ller(1987)]{muller1987weak}
H-G M{\"u}ller.
\newblock Weak and universal consistency of moving weighted averages.
\newblock \emph{Periodica Mathematica Hungarica}, 18\penalty0 (3):\penalty0
  241--250, 1987.

\bibitem[Ryabko and Bartlett(2006)]{ryabko2006pattern}
Daniil Ryabko and Peter Bartlett.
\newblock Pattern recognition for conditionally independent data.
\newblock \emph{Journal of Machine Learning Research}, 7\penalty0 (4), 2006.

\bibitem[Steinwart et~al.(2009)Steinwart, Hush, and
  Scovel]{steinwart2009learning}
Ingo Steinwart, Don Hush, and Clint Scovel.
\newblock Learning from dependent observations.
\newblock \emph{Journal of Multivariate Analysis}, 100\penalty0 (1):\penalty0
  175--194, 2009.

\bibitem[Stone(1977)]{stone1977consistent}
Charles~J Stone.
\newblock Consistent nonparametric regression.
\newblock \emph{The annals of statistics}, pages 595--620, 1977.

\bibitem[Urner and Ben-David(2013)]{urner2013probabilistic}
Ruth Urner and Shai Ben-David.
\newblock Probabilistic lipschitzness a niceness assumption for deterministic
  labels.
\newblock In \emph{Learning Faster from Easy Data-Workshop@ NIPS}, volume~2,
  page~1, 2013.

\end{thebibliography}

\appendix
\section{Proof of Proposition \ref{prop:binary_izi}}\label{ap:proof_binary_izi}
\begin{proof}
Let $y^0,y^1\in \Ycal$ such that $\ell(y^0,y^1):=\delta>0$. It suffices to observe that measurable functions $\Xcal\to\{0,1\}$ can be mapped to the measurable functions $\Xcal\to\{y^0,y^1\}$ by composing with the simple mapping $\phi$ such that $\phi(i)=y^i$ for $i\in \{0,1\}$. Consider a sequence $\Xbb\in \textnormal \suol_{(\Ycal,\ell)}$ and let $f_{\boldsymbol{\cdot}}$ be a universal learner for $\Xbb$, we will show that $\Xbb\in \suol_{(\{0,1\},\ell_{01})}$ by using this learner to perform binary classification. We define the learning rule $\hat f_{\boldsymbol{\cdot}} = (\hat f_t)_{t\geq 1}$  as follows, for any $x_{\leq t}\in \Xcal^{t}$ and $y_{<t}\in \{0,1\}^{t-1}$,
\begin{equation*}
    \hat f_t(x_{<t},y_{<t},x_t):= \begin{cases}
    0 & \text{if } \ell(f_t(x_{<t},\phi(y)_{<t},x_t),y^0)\leq \ell(f_t(x_{<t},\phi(y)_{<t},x_t),y^1)\\
    1 & \text{otherwise}.
    \end{cases}
\end{equation*}
where we used the notation $\phi(y):=(\phi(y_t))_{t\geq 1}$. Note that by relaxed triangle inequality,  
\begin{align*}
    \ell_{01}(\hat f_t(x_{<t},y_{<t},x_t),y_t) &\leq \1[\ell(f_t(x_{<t},\phi(y)_{<t},x_t),\phi(y_t))\geq \ell(f_t(x_{<t},\phi(y)_{<t},x_t),\phi(1-y_t))]\\
    &\leq \1[\ell(f_t(x_{<t},\phi(y)_{<t},x_t),\phi(y_t))\geq \frac{c_\ell}{2}\delta]\\
    &\leq \frac{2}{ c_\ell \delta} \ell(f_t(x_{<t},\phi(y)_{<t},x_t),\phi(y_t)).
\end{align*}
Then, for any measurable function $f^*:\Xcal\to\{0,1\}$ we have $\Lcal_{\Xbb}^{(\{0,1\},\ell_{01})}(\hat f_{\boldsymbol{\cdot}},f^*)\leq \frac{2}{ c_\ell \delta} \Lcal_{\Xbb}^{(\Ycal,\ell)}(f_{\boldsymbol{\cdot}},\phi\circ f^*)$, which by universal consistency of $f_{\boldsymbol{\cdot}}$ shows that $\Lcal_{\Xbb}^{(\{0,1\},\ell_{01})}(\hat f_{\boldsymbol{\cdot}},f^*)=0$ almost surely. Hence, $\hat f_{\boldsymbol{\cdot}}$ is a universal learner for the process $\Xbb$ for the setting $(\{0,1\},\ell_{01})$ i.e. $\Xbb\in \textnormal\suol_{(\{0,1\},\ell_{01})}$.
\end{proof}

\section{Proof of Lemma \ref{lemma:finite_classification}\label{ap:lemma_finite_classification}}
\begin{proof}
By Proposition \ref{prop:binary_izi}, it suffices to prove that any process $\Xbb\in \suol_{(\{0,1\},\ell_{01})}$ admits universal learning in the setup $([k],\ell_{01})$. To learn an unknown function $f^*:\Xcal\to[k]$, it suffices to learn the $k$ individual binary functions which predict each class: $f^{*,i}(\cdot):= \1(f^*(\cdot)=i)$ where $i\in [k]$. Given a universal learner $f_{\boldsymbol{\cdot}}$ for $\Xbb$ for binary classification, We can therefore consider a universal learner for $k-$multiclass classification $\hat f_{\boldsymbol{\cdot}}$ which follows the prediction of $f_{\boldsymbol{\cdot}}$ for all functions $f^i$ as follows: for any $x_{\leq t}\in \Xcal^t$ and $y_{<t}\in [k]^{t-1}$ we pose $\hat f_t(x_{<t},y_{<t},x_t):= \argmax_{1\leq i\leq k} f_t(x_{<t},\1(y=i)_{<t},x_t)$ where $\1(y=i)_{<t}$ denotes the sequence $\1(y_{t'}=i)_{t'\leq t}$. We can note that this learner makes a mistake only if $f_{\boldsymbol{\cdot}}$ made a mistake in the prediction of at least one of the functions $f^{*,i}$ for $1\leq i\leq k$. Thus,
\begin{equation*}
    \Lcal_{\Xbb}^{([k],\ell_{01})}(\hat f_{\boldsymbol{\cdot}},f^*) \leq \sum_{i=1}^k \Lcal_{\Xbb}^{(\{0,1\},\ell_{01})}(f_{\boldsymbol{\cdot}},f^{*,i}).
\end{equation*}
Then, $\Lcal_{\Xbb}^{([k],\ell_{01})}(\hat f_{\boldsymbol{\cdot}},f^*)=0$ almost surely by universal consistence of $f_{\boldsymbol{\cdot}}$ which shows that $\hat f_{\boldsymbol{\cdot}}$ is optimistically universal for $\Xbb$ and $k-$multiclass classification.
\end{proof}

\section{Proof of Theorem \ref{thm:invariance_optimistic}}
\label{appendix:invariance_optimistic}

\begin{proof}[of Theorem \ref{thm:invariance_optimistic}]
We start by supposing that there exists an optimistically universal learning rule $f_{\boldsymbol{\cdot}}^{\{0,1\}}$ for the binary classification setting, and now construct an optimistically universal learning rule for a general setting $(\Ycal,\ell)$ satisfying $0<\bar \ell<\infty$. This results from the fact that the construction in the proofs of both Theorem \ref{thm:binary_countable} and Theorem \ref{thm:reduction_countable} are invariant to $\Xbb$. Precisely, we first construct an optimistically universal learning rule for countably-many classification as given in the proof of Theorem \ref{thm:reduction_countable}. With
\begin{multline*}
    p^t(x_{<t},y_{<t},x_t;i) :=\frac{1}{2}\left( \Pbb_{\sigma}\left[f_t^{\{0,1\}}(x_{<t},\1(y\in \sigma)_{<t},x_t) = 1 \mid  i\in \sigma \right] \right.\\+ \left.\Pbb_{\sigma}\left[f_t^{\{0,1\}}(x_{<t},\1(y\in \sigma)_{<t},x_t) = 0 \mid  i\notin\sigma \right]\right),
\end{multline*}
we define
\begin{equation*}
    f_t^\Nbb(x_{<t},y_{<t},x_t):= \begin{cases}
    \displaystyle{ \min_{i\in \Nbb}} \left\{i,\; p^t(x_{<t},y_{<t},x_t;i)>\frac{3}{4} \right\} & \text{if }\exists i\in \Nbb,\; p^t(x_{<t},y_{<t},x_t;i)>\frac{3}{4} \\
    0 & \text{otherwise}.
    \end{cases}
\end{equation*}
By construction, and Theorem \ref{thm:binary_countable}, $f_{\boldsymbol{\cdot}}^\Nbb$ is an optimistically universal learning rule for $(\Nbb,\ell_{01})$. We now use the construction given by Theorem \ref{thm:reduction_countable} to get an optimistically universal learning rule $f_{\boldsymbol{\cdot}}^{(\Ycal,\ell)}$ for $(\Ycal,\ell)$. Define a sequence $(y^i)_{i\geq 1}$ dense in $\Ycal$ with respect to $\ell$. For $k\geq 1$ and $\epsilon_k=2^{-k}$, we define the functions $h_k(y) = \inf\{i\geq 1,l(y^i,y)<\epsilon_k\}$ and construct the learning rules $f_{\boldsymbol{\cdot}}^k$ by
\begin{equation*}
    f^k_t(x_{< t},y_{<t},x_t) = y^{f^\Nbb_t(x_{< t},h_k(y_{<t}),x_t)}.
\end{equation*}
Denoting by $B_\ell(y,\epsilon)=\{y'\in \Ycal, \ell(y,y')<\epsilon\}$ and $\Bcal_i^k:= B_{\ell}(y_i,\epsilon_k)\setminus \bigcup_{1\leq j<i} B_{\ell}(y_i,\epsilon_k)$, we now define our final learning rule
\begin{equation*}
    f_t^{(\Ycal,\ell)}(x_{\leq t},y_{<t},x_t) = f^{\hat p}_t(x_{\leq t},y_{<t},x_t) \text{ for }\hat p=\max\left \{1\leq  p\leq t, \bigcap_{1\leq k\leq p} \Bcal_{f^k_t(x_{\leq t},y_{<t},x_t)}^k \neq \emptyset \right\},
\end{equation*}
which is invariant to the process $\Xbb$, hence optimistically universal by the proof of Theorem \ref{thm:reduction_countable}.

We now show the converse. Suppose there exists some setting $(\Ycal,\ell)$ with $0<\bar\ell<\infty$ admitting an optimistically universal learner $f_{\boldsymbol{\cdot}}^{(\Ycal,\ell)}$. We will construct an optimistically universal learning rule for binary classification using the proof of Proposition \ref{prop:binary_izi}. Let $y^0,y^1\in \Ycal$ such that $\ell(y^0,y^1)>0$ and consider the function defined by $\phi(i)=y^i$ for $i\in\{0,1\}$. We now construct a learning rule $f_{\boldsymbol{\cdot}}^{\{0,1\}}$ for binary classification as follows
\begin{equation*}
    f_t^{\{0,1\}}(x_{<t},y_{<t},x_t):= \begin{cases}
    0 & \text{if } \ell(f_t^{(\Ycal,\ell)}(x_{<t},\phi(y)_{<t},x_t),y^0)\leq \ell(f_t^{(\Ycal,\ell)}(x_{<t},\phi(y)_{<t},x_t),y^1)\\
    1 & \text{otherwise}.
    \end{cases}
\end{equation*}
This learning rule is invariant to $\Xbb$, hence optimistically universal by the proof of Proposition \ref{prop:binary_izi}. This ends the proof of the theorem.
\end{proof}

\section{Additional background}
\label{appendix:alternative_reduction}
In the core of the paper, we presented the two inclusions $\suol_{(\Nbb,\ell_{01})} \subset \suol_{(\Ycal,\ell)}\subset \suol_{(\{0,1\},\ell_{01})}$ shown in \cite{hanneke2021learning} for general bounded loss settings $(\Ycal,\ell)$ (Prop \ref{prop:binary_izi} and Theorem \ref{thm:reduction_countable}). The results of \cite{hanneke2021learning} offer more details which are not useful for this paper but give perspective on previous state of the art as well as useful intuitions. Specifically, the set $\suol_{(\Ycal,\ell)}$ only depends on whether the value space $(\Ycal, \ell)$ is \emph{totally bounded}. We say that $(\Ycal,\ell)$ is totally bounded if it can be covered by a finite number of $\epsilon-$balls, i.e. $\forall \epsilon>0, \exists \Ycal_{\epsilon}\subset \Ycal$ s.t. $\#\Ycal_{\epsilon} <\infty$ and $\sup_{y\in \Ycal}\inf_{y\in \Ycal_\epsilon}\ell(y_\epsilon,y)\leq \epsilon$. Note that $(\{0,1\},\ell_{01})$ is totally bounded whereas $(\Nbb,\ell_{01})$ is not. \cite{hanneke2021learning} proved that any setup could be reduced to these two cases.

\begin{theorem}[\cite{hanneke2021learning}]
\label{thm:alternative_reduction}
For any separable near-metric space $(\Ycal,\ell)$ with $0<\bar \ell<\infty$,
\begin{itemize}
    \item If $\Ycal$ is totally bounded, $\textnormal\suol_{(\Ycal,\ell)}= \textnormal\suol_{(\{0,1\},\ell_{01})}$,
    \item If $\Ycal$ is not totally bounded, $\textnormal\suol_{(\Ycal,\ell)}= \textnormal\suol_{(\Nbb,\ell_{01})}$.
\end{itemize}
\end{theorem}
%The proof lies on two main ideas: (1) because the state space is separable, we can cover it by a countable number of ``balls" with size at most $ \epsilon$ with respect to the loss function $\ell$ for arbitrary small $\epsilon>0$ and (2) if universal learning for binary classification is achievable, then universal learning is achievable for multiclass classification with finite number of classes.
We will now give some intuition on the first point, which reduces the totally bounded setting to $k-$multiclass classification for $k\geq 2$. Finite multiclass classification can then be reduced to binary classification through Lemma \ref{lemma:finite_classification}. It will be useful to keep in mind the proof technique of this reduction for our main result, though it will reveal insufficient to reduce $(\Nbb,\ell_{01})$ to binary classification.

\paragraph{Sketch of proof of Theorem \ref{thm:alternative_reduction}.} By Theorem \ref{thm:reduction_countable}, we know that for any general setting, $(\Ycal, \ell)$ we have $\textnormal\suol_{(\Nbb,\ell_{01})} \subset \textnormal\suol_{(\Ycal,\ell)}$. The question is now, in which cases can we further reduce the setting to binary classification? Assume that in the construction of the proof of Theorem \ref{thm:reduction_countable}, the partition $(\Bcal_i^\epsilon)_{i\geq 1}$ of $\Ycal$ into balls of size at most $\epsilon>0$ can always be made finite. Then, we are able to construct an universally consistent learning rule from universally consistent rules for finitely-many classification, which is equivalent to universal consistence for binary classification by Lemma \ref{lemma:finite_classification}. Thus, we obtain the alternative $\textnormal\suol_{(\Ycal,\ell)}= \textnormal\suol_{(\{0,1\},\ell_{01})}$.

If this is not the case, there exists $\epsilon>0$ and an infinite---countable--- sequence $(y^k)_{k\geq 1}$ in $\Ycal$ which is $\epsilon-$separated i.e. such that $\ell(y^i,y^j)\geq\epsilon$ for any $i\neq j$. Using the mapping $\phi:\Nbb\to \Ycal$ defined by $\phi(i)=y^k$ for all $k\geq 1$ similarly to the construction in the proof of Proposition \ref{prop:binary_izi}, from a universal learner $f_{\boldsymbol{\cdot}}$ for $(\Ycal,\ell)$ we construct a learning rule $\hat f_{\boldsymbol{\cdot}}$ for $(\Nbb,\ell_{01})$, such that for any measurable function $f^*:\Xcal\to\Nbb$,
\begin{equation*}
     \Lcal_{\Xbb}^{(\Nbb,\ell_{01})}(\hat f_{\boldsymbol{\cdot}},f^*)\leq \frac{2}{ c_\ell \epsilon} \Lcal_{\Xbb}^{(\Ycal,\ell)}(f_{\boldsymbol{\cdot}},\phi\circ f^*),
\end{equation*}
which shows that almost surely, $\Lcal_{\Xbb}^{(\Nbb,\ell_{01})}(\hat f_{\boldsymbol{\cdot}},f^*)=0$. Therefore, any sequence which admits universal learning for $(\Ycal, \ell)$ must admit universal learning for $(\Nbb,\ell_{01})$ i.e. $\textnormal\suol_{(\Ycal,\ell)}\subset \textnormal\suol_{(\Nbb,\ell_{01})}$. This ends the alternative of the theorem.
\end{document}